\documentclass[10pt,twocolumn,letterpaper]{article}

\usepackage[pagenumbers]{cvpr} 

\usepackage{graphicx}
\usepackage{amsmath}
\usepackage{amsthm}
\usepackage{amssymb}
\usepackage{bbm}
\usepackage{booktabs}
\usepackage{color}
\usepackage{multirow}
\usepackage{caption}
\usepackage{subcaption}
\usepackage[linesnumbered,ruled,vlined]{algorithm2e}
\newtheorem{theorem}{Theorem}
\newtheorem{lemma}[theorem]{Lemma}

\DeclareMathOperator{\Tr}{Tr}

\usepackage[pagebackref,breaklinks,colorlinks]{hyperref}

\usepackage[capitalize]{cleveref}
\crefname{section}{Sec.}{Secs.}
\Crefname{section}{Section}{Sections}
\Crefname{table}{Table}{Tables}
\crefname{table}{Tab.}{Tabs.}

\def\cvprPaperID{10553} 
\def\confName{CVPR}
\def\confYear{2023}

\begin{document}


\title{Feature Correlation-guided Knowledge Transfer for Federated \\ Self-supervised Learning}



\author{%
  Yi~Liu$^{1}$,
  Song~Guo$^{1}$,
  Jie~Zhang$^{1}$,
  Qihua~Zhou$^{1}$,
  Yingchun~Wang$^{1}$,
  and~Xiaohan~Zhao$^{2}$\\
  \textsuperscript{1}Department of Computing, The Hong Kong Polytechnic University\\
  \textsuperscript{2}Institute for Applied Computational Science, Harvard University\\
  \texttt{\{joeylau.liu,qi-hua.zhou,daenerys.wang\}@connect.polyu.hk},\\
  \texttt{\{song.guo,jie-comp.zhang\}@polyu.edu.hk}, \\ \texttt{xiaohanzhao@g.harvard.edu}\\
}

\maketitle

\begin{abstract}
\vspace{-0.4cm}
    
    To eliminate the requirement of fully-labeled data for supervised model training in traditional Federated Learning (FL),  
    extensive attention has been paid to the application of Self-supervised Learning (SSL) approaches on FL to tackle the label scarcity problem.
    Previous works on Federated SSL generally fall into two categories: parameter-based model aggregation (i.e., FedAvg, applicable to homogeneous cases) or data-based feature sharing (i.e., knowledge distillation, applicable to heterogeneous cases) to achieve knowledge transfer among multiple unlabeled clients.
    Despite the progress, all of them inevitably rely on some assumptions, such as homogeneous models or the existence of an additional public dataset, which hinder the universality of the training frameworks for more general scenarios.
    Therefore, in this paper, we propose a novel and general method named \underline{\textbf{Fed}}erated Self-supervised Learning with \underline{\textbf{F}}eature-c\underline{\textbf{o}}rrelation based \underline{\textbf{A}}ggregation (FedFoA) to tackle the above limitations 
    in a communication-efficient and privacy-preserving manner.
    Our insight is to utilize feature correlation to align the feature mappings and calibrate the local model updates across clients during their local training process. More specifically, we 
    design a factorization-based method to extract the cross-feature relation matrix from the local representations.
    Then, the relation matrix can be regarded as a carrier of semantic information to perform the aggregation phase.
    We prove that FedFoA is a model-agnostic training framework and can be easily compatible with state-of-the-art unsupervised FL methods. 
    Extensive empirical experiments demonstrate that our proposed approach outperforms the state-of-the-art methods by a significant margin. 

\end{abstract}
\vspace{-0.4cm}
\section{Introduction}
\label{sec:intro}
\begin{figure}[t]
    \centering
    \setlength{\belowcaptionskip}{0mm} 
    \includegraphics[width=0.85\linewidth]{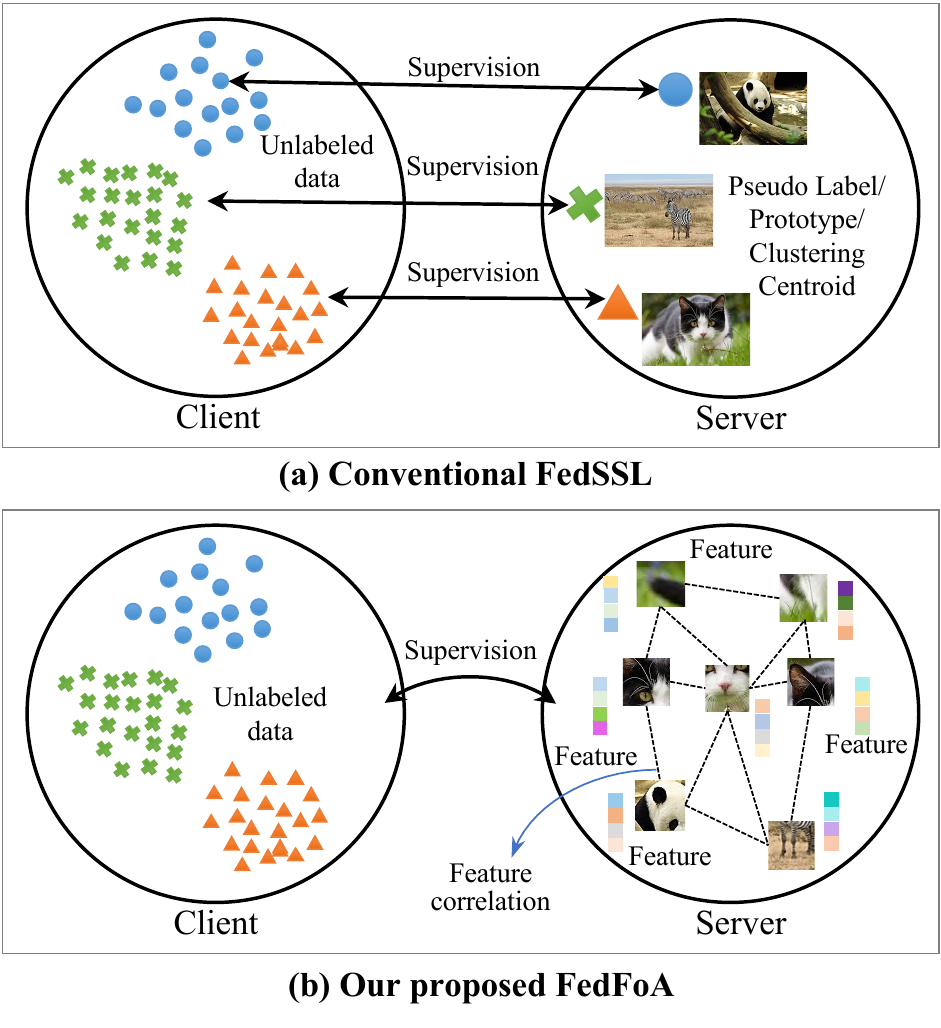}
    \caption{Comparison with conventional FedSSL methods. (a). In conventional FedSSL, knowledge transfer usually relies on local data information sharing (i.e., with the supervision of pseudo label \cite{jeong2020federated}, prototype \cite{zhang2020federated}, clustering centroid \cite{lubana2022orchestra}); While in (b), we focus on utilizing feature-correlation to transfer knowledge among different clients.}
    \label{fig:comparision}
    \vspace{-0.4cm}
\end{figure}

To deal with the privacy issues in distributed machine learning, Federated Learning (FL) \cite{mcmahan2017communication,zhang2021edge} has emerged as a promising paradigm by collaboratively training machine learning models without exposing the raw data of each client. 
As a common practice in FL, the global model is usually obtained by periodically averaging the updated model parameters from the clients in a centralized server. 
While existing FL methods assume that 
data is fully labeled so that supervised learning can be used for the model update on each client, in some real-world applications, labeling all the data is usually unrealistic due to high labor costs and the requirement of expert knowledge. For example, in medical diagnosis, the patient would be reluctant to contribute to the labeling process for privacy concerns \cite{ng2021federated}.

In this case, Self-supervised Learning (SSL) can be regarded as a promising paradigm to tackle the label scarcity problem in FL. With the concept of learning remarkable representations on a set of unlabeled data, SSL allows higher-level data representations to be extracted from raw data using deep neural networks, such that the learned representations can be easily used for downstream tasks \cite{van2020towards}. 
Many prior works have applied SSL to FL by directly extending unsupervised learning approaches to decentralized settings, where only unlabeled data are available at each client,
such as BYOL \cite{komodakis2018unsupervised}, SimSam \cite{chen2021exploring}. Considering the statistical characteristics of decentralized data, i.e., non-independent and identically distributed (Non-IID) data, FedU \cite{zhuang2021collaborative}, FedReID\cite{zhuang2020performance}, and FedEMA \cite{zhuang2021divergence} propose to dynamically control the knowledge transfer between global model and local models, while FedCA \cite{zhang2020federated} focuses on sharing features on an additional public dataset to reduce the divergence among different local models.

However, we found that these solutions above still fail to achieve the desired performance due to the following two limitations: 
\textit{1) inconsistency of feature representations}. In FL environment, limited unlabeled training data in each client may lead to overfitting of local models, resulting in discrepancy in feature representations from client to client. Worse still, Non-IID data would aggravate the inconsistency between representations without unified information among clients.
Although some previous works conduct the feature alignment by introducing an auxiliary public dataset for feature sharing (i.e., Fig.~\ref{fig:comparision}(a)), the selection of a public dataset may bring new challenges such as communication overhead.
Therefore, it is critical to design a data-free knowledge transfer scheme to calibrate the divergence among local models.
\textit{2) Dependence on homogeneous model aggregation}. Most of works rely on uploading the local models to perform knowledge aggregation, which requires all clients to hold homogeneous model architecture. This setting greatly limits the universality of the training framework in various scenarios, especially for clients with heterogeneous neural architectures.

To tackle these limitations, in this paper, we propose a novel and general method named \underline{\textbf{Fed}}erated Self-supervised Learning with \underline{\textbf{F}}eature-c\underline{\textbf{o}}rrelation based \underline{\textbf{A}}ggregation (FedFoA) to achieve data-free and model-agnostic collaborative training framework in a communication-efficient and privacy-preserving manner.
The key insight of FedFoA is to utilize feature correlation to align the feature mappings and calibrate the local model updates across clients during their local training process (i.e., Fig.~\ref{fig:comparision}(b)). More specifically, we design a factorization-based method to extract the cross-feature relation matrix from the local representations.
Then, the relation matrix can be regarded as a carrier of semantic information to perform the aggregation phase.
We prove that FedFoA is a model-agnostic training framework and can be easily compatible with state-of-the-art unsupervised FL methods. 
Extensive empirical experiments on various benchmarks and downstream tasks demonstrate that FedFoA outperforms the state-of-the-art methods by a significant margin in terms of model accuracy and communication efficiency. 
The contributions of the paper are summarized as follows:


\begin{itemize}
\item To the best of our knowledge, we are the first to explore the shared cross-feature correlation among different clients, and explicitly reveal the benefits of feature-correlation-based knowledge transfer in federated self-supervised learning.
\item We design a factorization-based method to extract shared feature correlation from local data and use it as a medium for knowledge transfer.
\item We propose a general training framework to align feature representations among different clients so as to perform unbiased knowledge aggregation process.
\item Extensive experiments on three typical image classification tasks and two different data settings show the superior performance of FedFoA over the state-of-the-art approaches.
\end{itemize}

\section{Related Work}\label{sec:related_work}

In this section, we summarize the prior related work in three aspects: federated learning, self-supervised learning and federated self-supervised learning.  

\subsection{Federated Learning}
Federated learning is a distributed machine learning paradigm that allows multiple clients to collaboratively train one global model through a parameter server without sharing their raw data. In the classic algorithm FedAvg\cite{mcmahan2017communication}, the global model parameters are calculated by a weighted average of local clients' models, in which the weight of a local model is positively correlated with the data volume. Another line of works focuses on client loss rather than data volume \cite{mohri2019agnostic, li2019fair}. Instead of directly averaging the global model, Wang \etal~\cite{wang2020federated} proposed to match neurons in a layer-wise manner before aggregation to achieve better convergence. 
Fedpara\cite{hyeon2021fedpara} proposed to utilize a factorization-based method to compress parameter size and reduce communication overhead of weights aggregation. 
However, all these weighted aggregation schemes rely on knowledge transfer via model parameters, which require all clients to hold homogeneous model architecture.
%
To further address model heterogeneity in FL, 
FedMD\cite{li2019FedMD} and FedDF\cite{lin2020ensemble} design distillation-based methods that use additional public datasets or generated data as the common input and aggregate the features among clients. Instead of introducing an additional dataset, FedProto\cite{tan2021fedproto} generates label-associated feature prototypes computed by the private datasets to perform knowledge transfer. However, this methodology is only applicable to supervised learning tasks, where labeled data are available for each client.

\subsection{Self-supervised Learning}

Self-supervised learning (SSL) aims to learn a general feature representation from unlabeled dataset. The preliminaries of SSL can be classified into two different categories: generative methods and discriminative methods. The principle of generative methods is to learn the representations via generating pixels of input \cite{vincent2008extracting, goodfellow2020generative}. 
While in discriminative methods, proxy tasks are used to guide representation learning \cite{gidaris2018unsupervised, noroozi2016unsupervised, pathak2016context, zhang2016colorful}. As one promising branch, contrastive learning-based methods \cite{oord2018representation, bachman2019learning} define two augmented samples of the same input as a positive pair, and two different samples as a negative pair. In the training process, they use contrastive loss to minimize the distance between positive pairs while maximizing the distance between negative pairs.  MoCo\cite{he2020momentum} generates the negative samples from a memory bank and SimCLR\cite{chen2020simple} find the negative pairs in large batch size. BYOL\cite{grill2020bootstrap} and SimSiam\cite{chen2021exploring} explore more efficient approaches that skip the negative pairs training and only focus on positive pairs.

\subsection{Federated Self-supervised Learning}
Federated self-supervised learning (FedSSL) is still a nascent topic with the goal of learning representations from decentralized unlabeled data while protecting data privacy. 
Some prior works on FedSSL develop a direct combination of self-supervised learning techniques with FL, such as
\cite{van2020towards,jin2020towards}. To further tackle the non-independent and identically distributed (Non-IID) data challenges in FL, 
clustering-based frameworks 
\cite{dennis2021heterogeneity,lubana2022orchestra}, auto-encoding-based method \cite{zhang2020federated}, and contrastive learning-based approaches (i.e., FedU \cite{zhuang2021collaborative} based on Siamese, FedReID \cite{zhuang2020performance}, FedX~\cite{han2022fedx} and FedAMA \cite{zhuang2021divergence} based on BYOL) are proposed to reduce the divergence in feature representations among different clients caused by heterogeneous local training processes.
%
However, all above methods require that all clients own homogeneous network architecture, which may be infeasible in some real-world applications. The federated self-supervised learning in the heterogeneous systems has not been explored yet. To the best of our knowledge, this work is the first attempt to tackle the model heterogeneity problem in the FedSSL setup.
\section{Methodology}

In this section, we explain the motivation and technical details of FedFoA. 
In federated unsupervised learning, though various clients may hold different views of valuable representations, the relationship between features is homogeneous since their dataset belong to a unit task (Fig.~\ref{fig:onecol}). Instead of directly aggregating the model weights, we align the features among clients from a new perspective: \textit{the semantic relationship between features}. 

\subsection{Problem Statement}
 

In this subsection, we first review the standard FedSSL problem.
%
%
The objective of federated self-supervised learning is to assemble the knowledge from multiple decentralized clients to learn general representations for downstream tasks. Supposing there are $N$ clients with heterogeneous neural networks in the federated self-supervised learning system. Each client holds an unlabeled dataset $\mathcal{D}_i = \{\mathcal{X}_i\}$. We uniformly express all the self-supervised learning loss functions, both contrastive and non-contrastive, as $f(\omega)$, where $\omega$ represents the model parameters of the online encoder. In federated learning, the global objective is to learn a general model parameter set that can fit the data from all participated clients. The global loss function can be expressed as:
\begin{equation}
    f(\omega) = \sum_{i=1}^N{D_i\over D} f_i(\omega_i),
\end{equation}
where $D_i$ represents the data volume of client $i$, $D=\sum_{i=1}^N D_i$. The goal of conventional federated learning is to solve the following optimization problem:
\begin{equation}
    \omega^* = \mathop {\arg \min }_\omega f(\omega).
\end{equation}

Taking contrastive loss as an example, in the local training of FedSSL, each client samples a mini-batch of data and uses data augmentation to generate two different views. The same image on different views will be taken as positive pair, and two different images will be regarded as negative pair. Suppose the batch size is $m$, the conventional contrastive loss is represented as:
\begin{equation}
    \ell_{c}(z_i,z_j) = -\log {{\exp(\text{sim}(z_i,z_j)/\tau)}\over{\sum_{k=1}^{2m} \mathbbm{1}_{k\neq i} \exp(\text{sim}(z_i, z_k)/\tau)}},
\end{equation}
where $\text{sim}(u,v)$ represents the cosine similarity between vector $u$ and $v$. $\mathbbm{1}\in \{0,1\}$ represents the positive pair indicator and $\tau$ denotes a temperature parameter.

However, in the heterogeneous federated learning system, it is impossible to maintain a shared unified global model. Each client might hold a unique model and the dimensions of the model parameters among clients may not match. As a result, parameter-based model aggregation becomes infeasible. Though the model weights' knowledge can not be transferred, the knowledge in feature correlation can still be transferred in a model-agnostic way. Our goal is to develop a novel method that can extract the feature correlation knowledge and use it to replace the weights aggregation in heterogeneous federated self-supervised learning.   
%


\begin{figure}[t]
    \centering
    \includegraphics[width=0.9\linewidth]{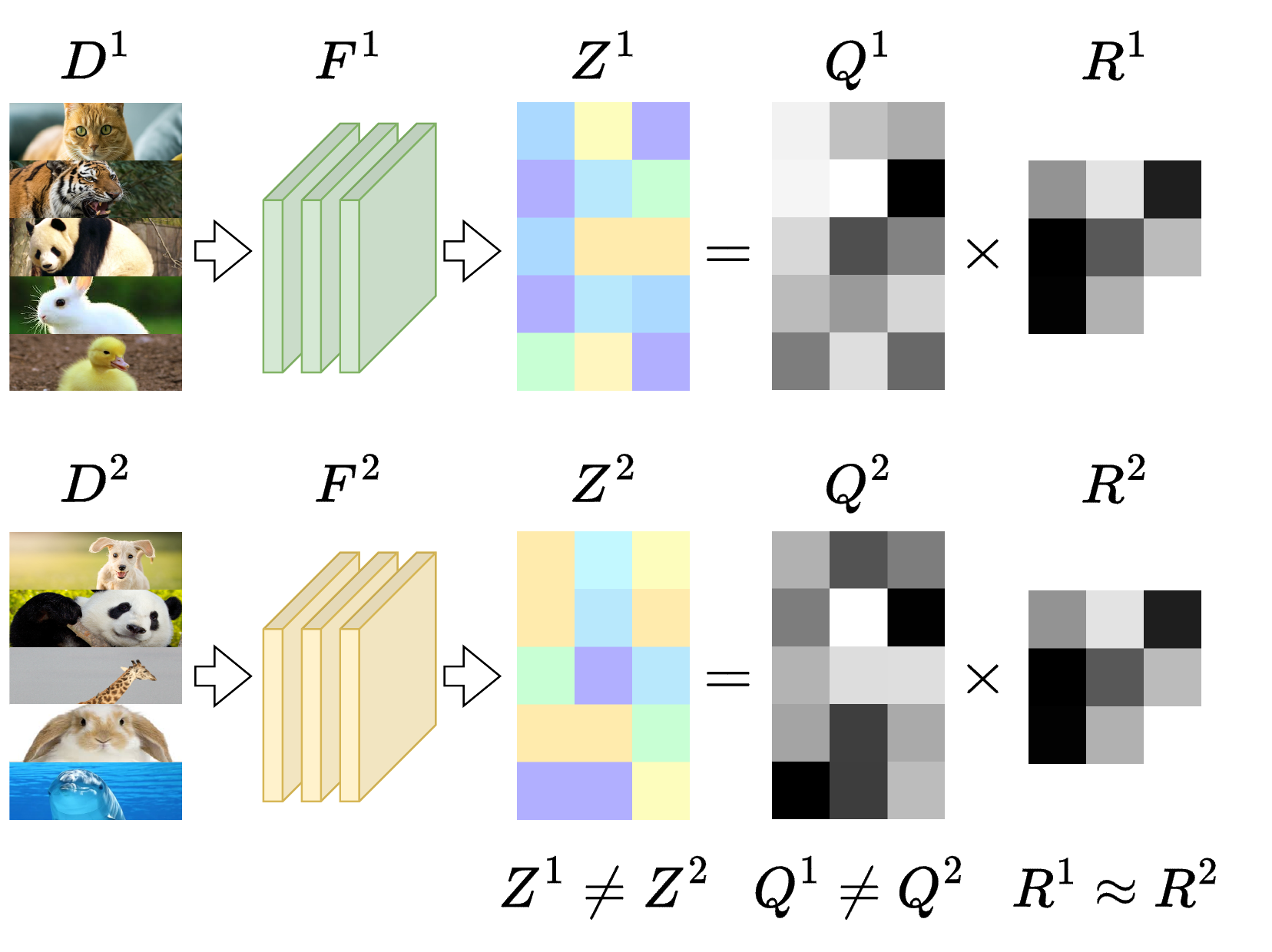}
    \caption{The feature-correlation can be extracted by using QR-based method \cite{gander1980algorithms}. For well trained models on the same task, their feature correlation tends to be similar.}
    \label{fig:onecol}
    \vspace{-0.2cm}
\end{figure}

\subsection{QR-based Feature-correlation Extraction} \label{QR}

In this section, we introduce the motivation and methodology of feature correlation extraction in FedFoA.

Like every color in the real-world can be represented as a linear combinations of red, green, and blue, we believe that the feature representation can also be factorized into semantic basis vectors and the linear coefficients that maps the basis vector space to the original feature vector space. 

To factorize the feature representation into the basis vectors and corresponding linear coefficients, we utilize a simple but effective factorization method, QR decomposition \cite{gander1980algorithms, sharma2013principal}. Though Singular Value Decomposition (SVD) can also complete this task, in our scenario, QR decomposition has three advantages: 1) QR decomposition is numerically stable and the solution is unique. By contrast, SVD has multiple factorization solutions over one matrix. It is difficult for clients to share knowledge via an unstable carrier. 2) QR decomposition is computationally efficient. QR decomposition has been applied in Principal Component Analysis (PCA)\cite{sharma2013principal} and proved to be more efficient than SVD-based method. 3) The $R$ matrix of QR composition has double physical meaning in our scenario. Except for the linear coefficients, $R$ also represents the linear correlations among original feature vector space. For distributed clients with the same learning task, we believe sharing feature correlations and utilizing it in the training process is beneficial. It can facilitate clients to learn a more general feature representations.



Suppose the size of client $i$'s dataset is $m$ and the dimension of feature representation is $n$, then the learned feature can be represented as a $m\times n$ matrix $Z^i_{t,b}$ ($m\gg n$), where $t$ denotes the current round's index and $b$ represents the batch-id. 
%
The feature can be further factorized to one orthogonal matrix $Q^i_{t,b}$ and one upper triangular matrix $R^i_{t,b}$:
\begin{equation}
    F_{qr}(Z^i_{t,b}) \to Q^i_{t,b}, R^i_{t,b},
\end{equation}
where the dimension of $Q^i_{t,b}$ is a $m\times n$ and $R^i_{t,b}$ is a $n\times n$ matrix. The QR decomposition extracts the basis vector $Q^i_{t,b}$ out of $Z^i_{t,b}$, and $R^i_{t,b}$ is the linear coefficients that maps the basis space back to the original feature space. Hence, $Z^i_{t,b}$, $Q^i_{t,b}$, $R^i_{t,b}$ satisfy $Z^i_{t,b}t = Q^i_{t,b}R^i_{t,b}$. 

It should be noticed that QR decomposition is highly sequentially biased. In the calculation process, it automatically sets the first column of the given matrix as the first basis vector. In the following iterations, it calculates the linearly independent part of current row from previous basis and then uses the result as the current row's basis. In other words, the $R$ matrix also represents the cross-feature relation of a given collection of feature representations. Each value in the diagonal of $R$ matrix represents the independence of corresponding feature vector. Since the QR decomposition is highly sequentially biased, the first feature vector is the most independent one. As a result, the value on the diagonal of $R$ follow a decreasing pattern. 

As the learning process continues, the independence of feature vectors increases. We use the trace of $R$ matrix as the measurement of the independence of current feature representation. In Appendix A, we empirically prove that the downstream performance is positively correlated with the independence of feature representation. 

\begin{figure*}[t]
    \centering
    \includegraphics[width=0.95\linewidth]{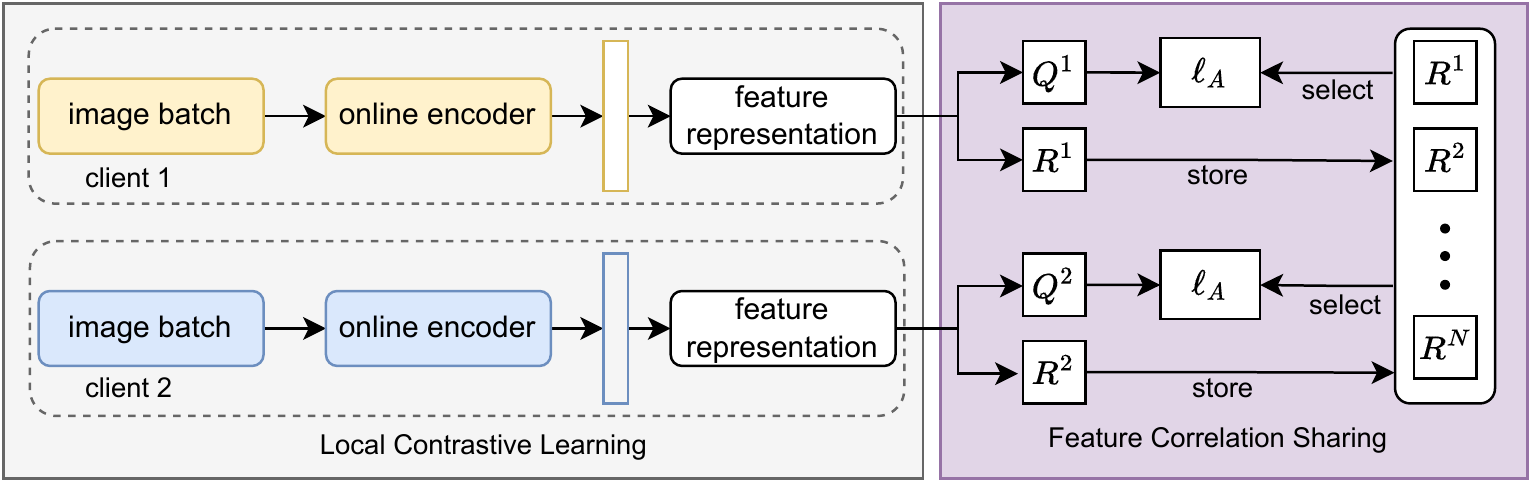}
    \caption{Illustration of FedFoA framework. FedFoA works in two seperated parts: 1). the contrastive local training; 2) the feature correlation knowledge sharing. FedFoA works in the second part. It only requires clients to share one of the factorized feature correlation matrix, preserving data privacy. The linear projector calibrates the features of heterogeneous network into one identical projection dimension. }
    \vspace{-0.4cm}
\end{figure*}

\subsection{Mutual Understanding Mechanism}
\label{mutual understanding}
After extracting the feature-correlation matrix by QR decomposition, the next step is to transfer knowledge through it. 
Instead of directly facilitating clients to learn a general cross-feature correlation matrix, we encourage each client to understand each other's cross-feature correlation. Motivated by this, we design a mutual-understanding mechanism to guide the training of clients.

Suppose client $i$ calculates it's own feature representation $Z^i$, and receives a feature correlation matrix $R^j$ from client $j$. If $i$ can recreate $Z^i$ based on $R^j$ and one basis vectors $Q$, we can say that $i$ understands $j$'s feature correlation. The optimal recreation can be formulated as:
\begin{equation}
\begin{aligned}
    & Q^* = \min_{Q} ||Z^i-QR^j||_2, \\
    & s.t. \quad  Q^\top Q = I
\end{aligned}
\end{equation}

This optimization problem can be solved by a auxiliary SVD decomposition:
\begin{equation}
    \begin{aligned}
         &Q^* = VU^\top, \\  
         &s.t. \quad   U, \Sigma, V = \text{SVD}(R^j{Z^i}^\top) \\
    \end{aligned}
\end{equation}

This optimization is proved in Lemma~\ref{lemma}. Once we have the optimization solution, each client can recreate its own feature representation based on the feature correlation matrix of others. A successful recreation means two clients have a good mutual understanding. Guided by this logic, we design a regularization term to facilitates clients to enhance their mutual understanding:
\begin{equation}
\label{loss}
    \ell_{A}(Z^i,R^j) = ||Z^i - Q^*R^j||
\end{equation}

The Regularization term $\ell_{A}$ can be regarded as a plug-and-play module 
that can be used standalone or stacked on top of existing FedSSL training frameworks.

\begin{lemma}
\label{lemma}
Given a $m\times n$ matrix $Z$ and a $n\times n$ matrix $R$, the optimal orthogonal
matrix $Q$ that minimizes $||QR-Z||_2$ is equal to $VU$, where $U$, $\Sigma$, $V$ is the SVD decomposed result of $RZ^\top$.
\end{lemma}

\begin{proof}
\begin{equation}
    \begin{aligned}
         &\mathop{\arg\min}||QR-Z||_2 \\
         \iff& \mathop{\arg\min}\Tr\left((QR-Z)(R^{\top}Q^\top-Z^\top)\right) \\
         \iff& \mathop{\arg\min}\Tr(QRR^\top Q^\top-ZR^\top Q^\top-QRZ^\top+ZZ^\top) \\
         \iff& \mathop{\arg\min}\Tr(RR^\top) + \Tr(ZZ^\top) - 2\Tr(Z^\top QR) \\
         \iff& \mathop{\arg\max}\Tr(RZ^\top Q)
    \end{aligned}
\end{equation}
Suppose we can use an auxiliary SVD decomposition to factorize $RZ^\top$ as $U$, $\Sigma$, $V$ that satisfy: $RZ^\top = U\Sigma V^\top$, $V^\top V=I$ and $U^\top U = I$. $\Sigma$ is a diagonal matrix of eigenvalues. 
Then the derivation of above optimization problem can be continued:
\begin{equation}
    \begin{aligned}
         &\mathop{\arg\max}\Tr(RZ^\top Q) \\
         \iff& \mathop{\arg\max} \Tr(U\Sigma V^\top Q) \\
         \iff& \mathop{\arg\max} \Tr(\Sigma V^\top QU) \\
         \iff& \mathop{\arg\max} \Tr(V^\top QU)
    \end{aligned}
\end{equation}
Since $V$, $Q$, $U$ are three orthogonal matrices, their product is also an orthogonal matrix. So we have following inequality:
\begin{equation}
    \Tr(V^\top QU) \leq \Tr(I) = 1
\end{equation}
The boundary condition that makes $V^\top QU = I$ is $Q = VU^\top$, so that we have the following conclusion:
\begin{equation}
    Q^* = VU^\top
\end{equation}
which completes the proof.
\end{proof}

\subsection{Round-wise Approximation}

We then explain the advantages of communication efficiency in this section. 
In the ideal setting, client $i$ only needs to upload its correlation matrix $R_t^i$ for every training batch. 
Though the feature correlation matrix $R_t^i$ has a relatively smaller shape than the original feature $Z_t^i$, the communication frequency is high since $R_t^i$ updates every batch. The communication overhead will become unbearable if we use $R_t^i$ to calculate our regularization term for every batch.

To further reduce the communication overhead, we use round-wise aggregation rather than batch-wise aggregation. Client $i$ calculates the local average correlation matrix in one training round $\Bar{R}_t^i$ by:
\begin{equation}
    \Bar{R}_t^i = {1\over B} \sum_{b=0}^B R^i_{t,b},
\end{equation}
where $B$ is the total batch number in one training round. Then $i$ updates $\Bar{R}_t^i$ to the memory bank as its representative correlation matrix. Since QR decomposition is numerically stable and the solution is unique, such an approximation does not undermine the semantic information in feature correlation. The experiment in Appendix C empirically proves that this round-wise approximation can save the traffic cost in a large volume while still maintain outstanding performance.

\begin{algorithm}[t]
    \SetKwInput{KwIn}{Input}
    \SetKwInput{KwOut}{Output}
    \SetKwFunction{CLIENT}{Client}
    \SetKwFunction{FOA}{FedFoA}
    \SetKwProg{Fn}{Function}{:}{}
	\caption{The workflow of FedFoA.}
    \label{algo}
        \textcolor{black}{\KwIn{learning rate $\eta$, total training rounds T, local batches B, local dataset $\mathcal{D}^0\sim \mathcal{D}^N$}
        \KwOut{model weights from clients $\omega^0 \sim \omega^N$}
        \For{\textup{training round} $t$ in $1,2,\cdots, T$}{
            \For{\textup{client} $i \in \mathcal{N}$}{
                \CLIENT{$\mathcal{D}^i, \mathcal{R}$}
            }
        }
        \Fn{\CLIENT{$\mathcal{D}^i, \mathcal{R}$}}{
            \For{\textup{batch} $b \in 1,2,\cdots, B$}{
                $Z_{t,b}^i \gets F(X_{t,b}^i, \omega^i)$  \\
                $Q_{t,b}^i, R_{t,b}^i \gets F_{qr}(Z_{t,b}^i)$ \\
                $\ell = \ell_{c}(Z_{t,b}^i)$ \\
                \If{\textup{round} $t>t_{warm}$}{
                    $\ell \gets $ \FOA{$\ell, Z_{t,b}^i, R_{t,b}^i,\mathcal{R}$}
                }
                $\omega^i \gets \omega^i - \ell'(\omega^i)$
            }
            $\Bar{R}_t^i \gets {1\over B}\sum_{b=0}^B R_{t,b}^i$\\
            update memory bank $\Bar{R}_t^i \to \mathcal{R} $
        }
        \Fn{\FOA{$\ell, Z_{t,b}^i,R_{t,b}^i,\mathcal{R}$}}{
            \For{\textup{client} $j \in \mathcal{N}$}{
                \If{$j\neq i$}{
                    $\Bar{R}_{t-1}^j \gets \mathcal{R}$ \\
                    \If{$\Tr(\Bar{R}_{t-1}^j)> \Tr(R_{t,b}^i)$}{
                        $U,\Sigma,V \gets \text{SVD}(\Bar{R}_{t-1}^j, {Z_{t,b}^i}^\top)$ \\
                        $Q^* = VU^\top$ \\
                        $\ell_{A} = ||Z_{t,b}^i - Q^*\Bar{R}_{t-1}^j||$ \\ 
                        $\ell \gets \ell + \lambda \ell_{A}$
                    }
                }
            }
            return $\ell$
        }}
\end{algorithm}

\subsection{FedFoA}


In this section, we explain the FedFoA framework in details. We summarize the workflow of FedFoA in Algorithm~\ref{algo}. The feature space among clients might vary in the heterogeneous federated learning system. To unify the dimension, we add one linear calibration layer at the tail part to calibrate the various feature dimension among clients into an identical projection space. In our framework, the conventional contrastive loss is also calculated based on the projection space.
In every training round, all clients calculate their correlation matrix and then upload them to the global memory bank $\mathcal{R}$. Since the semantic information might be chaotic in the first few training rounds, we set warm-up limits $t_{warm}$. The feature correlation knowledge transfer starts after the first $t_{warm}$ rounds. 
In the following training iterations, each client follows a Person-to-Person (P2P) protocol to access another client's feature correlation matrix. As mentioned in Sec.~\ref{QR}, the trace of the $R$ matrix can be used to measure the learning process. 
Once the chosen correlation matrix has a higher trace, client $i$ uses the chosen correlation matrix to calculate our regularization term by Eq.~\ref{loss}. We use a hyper-parameter $\lambda$ to control the trade-off between traditional contrastive loss and our regularization term. 
\begin{equation}
    \ell = \ell_{C} + \lambda \ell_{A}
\end{equation}





During our feature correlation knowledge transfer process, the local feature privacy is guaranteed since the shared part is only a $R$ matrix. In the QR decomposition, the $Q$ matrix is independent of the $R$ matrix. There is no feasible solution to calculate $Q$ by $R$. Without $Q$, it is impossible to derive the original local feature by $R$ matrix only. We can also understand it from a security perspective: the feature correlation knowledge is encoded as a private key $Q$ and a public key $R$. The feature correlation knowledge transfer can be achieved by only communicating the public key without accessing the private key. Hence, information privacy is guaranteed.

\section{Evaluation}

\subsection{Experiment Setup}

In this section, we provide some basic experiment settings, and the detailed version is provided in Appendix B.  
First, we set two scenarios to evaluate the performance: heterogeneous federated learning system and homogeneous federated learning system. Since FedFoA plays different roles in these two scenarios, we compare the performance under two experiment settings. 

\noindent\textbf{Heterogeneous setting:} We choose BYOL\cite{grill2020bootstrap}, FedMD\cite{li2019FedMD}, and FedDF \cite{lin2020ensemble} as our baselines. Though FedMD, and FedDF are proposed to solve the heterogeneous federated learning in supervised setting, their methodology are general solutions and do not rely on labeled public dataset.
In the heterogeneous setting, we use 4 clients with 4 different model architectures: ResNet-18, ResNet-34, AlexNet and VGG-9. 
After training, we use the downstream task to individually evaluate the accuracy of each client. And we use the mean accuracy as the evaluation criteria.

\noindent\textbf{Homogeneous setting:} We choose FedBYOL, FedU \cite{zhuang2021collaborative} and FedEMA \cite{zhuang2021divergence} as our baselines. FedFoA aims to transfer the feature correlation knowledge among clients. Hence, the benefit of FedFoA is orthogonal to baselines. We combine FedFoA with three baselines as integrated algorithms and demonstrate the improvement brought by FedFoA. In the homogeneous system, we set Resnet18 as the global model and evaluate the downstream task performance on the global model.

    
\noindent\textbf{Other general setting:} Our experiments are established on two datasets: CIFAR-10 and CIFAR-100. The data distribution follows the independent and identical distribution. We set 500 as the batch size, 0.032 as the learning rate, and 256 as the projection dimension. The downstream task follows a linear evaluation protocol, and we train a single linear layer of 200 epochs to evaluate the performance. In FedFoA, we set the warm up rounds as 5 and $\lambda$ as 0.01. All experiments are conducted on NVIDIA RTX 3090.



\subsection{Performance Evaluation}

In this section, we compare the performance of FedFoA under heterogeneous setting and homogeneous setting. 
We summarize the experiment result of CIFAR-10 and CIFAR-100 under heterogeneous federated learning under in Table~\ref{table_hetero_cifar10} and Table~\ref{table_hetero_cifar100}. We set the BYOL algorithm as the core contrastive learning algorithm for each client to do local training. Without any aggregation method, the performance among clients varies. The model architecture that has the highest down stream accuracy is ResNet-18 while the lowest is AlexNet. We calculate the average accuracy as the evaluation criterion. 
We apply FedMD, FedDF, and FedFoA as aggregation methods to improve their performance. 

\textbf{Heterogeneous Case.} As Table \ref{table_hetero_cifar10} shows, FedMD slightly improves the accuracy from 74.33\% to 74.41\%, and FedDF improves the accuracy to 75.55\%. FedFoA can improve the accuracy to 77.35\%. Hence, our FedFoA outperforms the other two baselines by a significant margin. The experiment on CIFAR-100 also demonstrates this conclusion. 

By aligning the feature representations among heterogeneous clients over a public dataset, FedMD successfully improves the accuracy of AlexNet and VGG-9 but undermines the performance of ResNet-18 and ResNet-34 on a small scale. Though the overall influence of FedMD is still positive, it shows the limitations of FedMD that it relies heavily on the quantity and quality of the public dataset. FedDF uses a collaboratively generated dataset as the carrier of knowledge to prevent the limitation of the public dataset. It aims to share knowledge at the feature level and force all clients to have a common feature representation. 

In our analysis, FedFoA has three advantages that make it outstanding: 1). FedFoA uses the knowledge embedded in the private dataset so that it does not require any public dataset. 2). Instead of assembling feature representations among heterogeneous clients, FedFoA facilitates clients to have a mutual understanding of each other's feature correlation. 3). To prevent the negative effect of the clients with lower performance, we filter them out of the training process by simple trace metrics of their correlation matrix. 

\textbf{Homogeneous Case.} The experiment result of homogeneous federated learning is shown in Table~\ref{homogeneous}. Since FedFoA focuses on feature correlation knowledge transferring, the benefits of FedFoA are orthogonal to other state-of-the-art approaches. Hence, in this section, we use BYOL-based federated self-supervised learning approaches as our baselines and evaluate the performance improvement brought by the regularization term of FedFoA. 

As Table~\ref{homogeneous} shows, FedFoA significantly improves the performance of all baselines. The best method is the integration with FedEMA and FedFoA, reaching 78.82\% in CIFAR-10 and 49.87\% in CIFAR-100. 
FedFoA brings marginal accuracy improvements of 2.7\% and 6.04\% over two tasks respectively. As demonstrated in Figure~\ref{acc}, the integrated algorithm learns faster than FedEMA solo. Reaching 40\% down stream accuracy costs the integrated algorithm 30 rounds of training. By contrast, it will take FedEMA for more than 50 training rounds. 

We evaluate different trade-off settings between conventional contrastive loss and our regularization term by controlling the value of $\lambda$. As Figure~\ref{lambda} shows, the most suitable value of $\lambda$ under CIFAR-100 task is 0.01. 

\begin{figure}
    \centering
    \setlength{\abovecaptionskip}{0mm} 
    \includegraphics[width=0.85\linewidth]{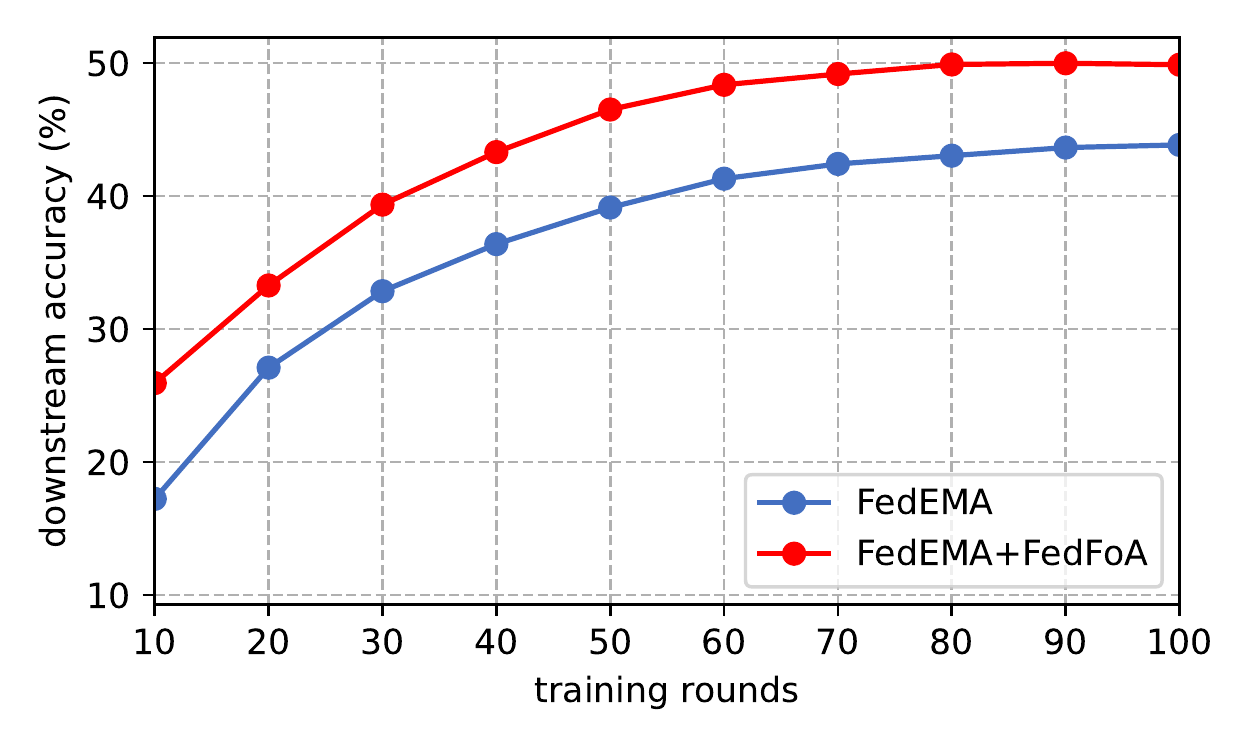}
    \caption{Downstream accuracy of FedEMA and an integrated algorithm of FedEMA and FedFoA under CIFAR-100 task.}
    \label{acc}
\end{figure} 

\begin{figure}
    \centering
    \setlength{\abovecaptionskip}{0mm} 
    \includegraphics[width=0.85\linewidth]{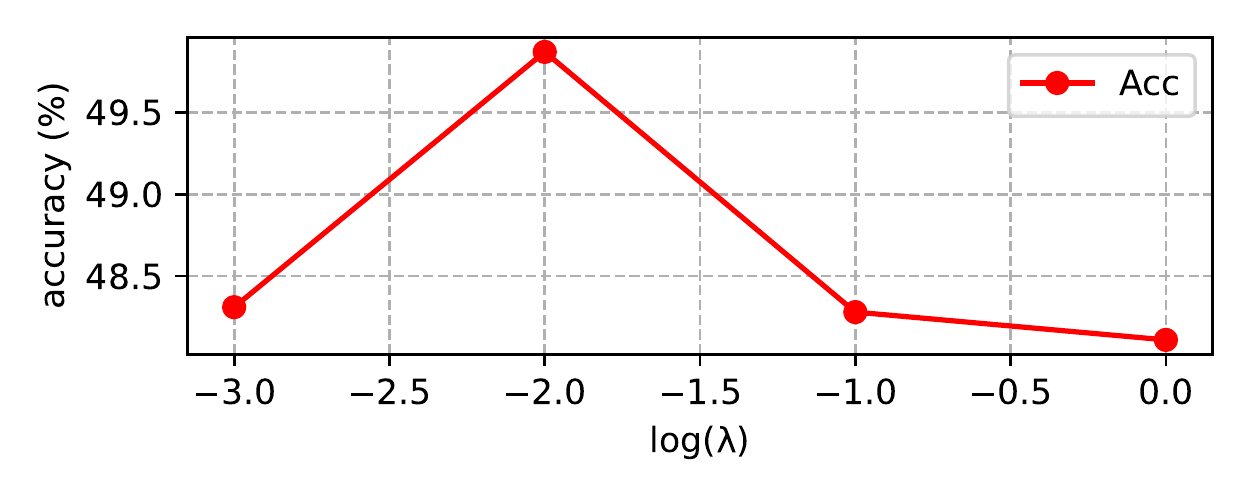}
    \caption{Influence of hyper-parameter $\lambda$. The experiment is conducted via integrated algorithm with FedEMA and FedFoA under CIFAR-100 dataset.}
    \label{lambda}
\end{figure} 

\begin{table}[h]
\centering
\caption{The performance evaluation in CIFAR-10 under the heterogeneous federated learning setting.}
\scalebox{0.85}{
\begin{tabular}{lccccc}
\toprule
\multirow{2}{*}{Method} & 
\multicolumn{4}{c}{Architecture}  & \multirow{2}{*}{Avg.} \\
\cmidrule{2-5}
& ResNet-18 & VGG-9 & AlexNet & ResNet-34 &                          \\
                        \midrule \midrule
BYOL                    & 77.93                            & 73.49                       & 68.74                           & 77.14        & 74.33                        \\ 
FedMD                   & 76.45                             & 74.32                         & 70.13                            & 76.74        & 74.41                        \\ 
FedDF                   & 78.35                             & 74.54                          & 70.89                            & 78.43        & 75.55                        \\ 
\midrule
\textbf{FedFoA}                  & \textbf{80.45}                            & \textbf{75.86}                         & \textbf{71.95}                          & \textbf{81.14}        & \textbf{77.35}                        \\ \bottomrule
\end{tabular}}
\label{table_hetero_cifar10}
\end{table}

\begin{table}[t]
\centering
\caption{The performance evaluation in CIFAR-100 under the heterogeneous federated learning setting. }
\scalebox{0.85}{
\begin{tabular}{lccccc}
\toprule
\multirow{2}{*}{Method} & 
\multicolumn{4}{c}{Architecture}  & \multirow{2}{*}{Avg.} \\
\cmidrule{2-5}
& ResNet-18 & VGG-9 & AlexNet & ResNet-34 &                          \\
                        \midrule \midrule
BYOL                    & 50.61                            & 45.59                       & 39.42                           & 49.52        & 46.29                        \\ 
FedMD                   & 50.37                             & 45.98                         & 41.2                            & 49.17        & 46.68                        \\ 
FedDF                   & 51.35                             & 46.24                          & 41.31                            & 49.87        & 47.19                        \\ 
\midrule
\textbf{FedFoA}                  & \textbf{52.85}                            & \textbf{42.06}                         & \textbf{50.49}                          & \textbf{50.49}        & \textbf{48.07}                        \\ \bottomrule
\end{tabular}}
\label{table_hetero_cifar100}
\end{table}

\begin{table}[t]
\centering
\caption{The performance evaluation under the homogeneous federated learning setting.}
\begin{tabular}{l|cc|cc}
\toprule
\multirow{2}{*}{Method} & \multicolumn{2}{c|}{CIFAR-10} & \multicolumn{2}{c}{CIFAR-100}  \\
                        & Last          & Best          & Last           & Best                  \\ 
                        \midrule\midrule
FedBYOL                 & 73.76             & 74.21             & 42.56              & 42.94               \\ 
\midrule
FedBYOL+FedFoA          & 78.54             & \textbf{79.00}             & 46.89              & \textbf{47.13}               \\ 
\midrule
FedU                    & 77.12             & 77.24             & 44.85              & 45.17                   \\ 
\midrule
FedU+FedFoA             & 78.37             & \textbf{78.71}             & 47.29              & \textbf{47.46}              \\ 
\midrule
FedEMA                  & 75.48             & 76.12             & 43.19              & 43.83                      \\ 
\midrule
FedEMA+FedFoA           & 78.82             & \textbf{78.82}             & 49.84             & \textbf{49.87}           \\
\bottomrule
\end{tabular}
\label{homogeneous}
\end{table}

    

%

\subsection{Visualization}

In this section, we use visualization methods to show the intermediate effects and final feature representation of FedFoA.  

As mentioned in Sec.~\ref{mutual understanding}, FedFoA is designed to enhance mutual understanding among clients. Though FedFoA does not require all clients to share a global feature correlation matrix, the regularization term in Eq.~\ref{loss} still can facilitate clients to learn similar feature correlations. To further prove this, we track the distance of the $R$ matrix among clients in the training process. 
We visualize the result using a heat map in Figure~\ref{heatmap}. The color and number inside each box represent the Euclidean distance of two feature correlation matrix. 
This intermediate result is recorded in the heterogeneous federated system. The number on the x axis and y axis represents the client id. Their model architectures are ResNet-18, VGG-9, AlexNet and ResNet-34 respectivelly.

In the first round, the feature correlation distance is high among clients. The distance between ResNet-18 and ResNet-34 is relatively shorter since they share the same basic block. Their semantic information tends to be similar. With the help of FedFoA, the feature correlation distance keeps decreasing 
in the subsequent training rounds. After convergence, though the most similar pair is still ResNet-18 and ResNet-34, their similarity is no longer significant among clients. This intermediate representation shows that FedFoA successfully narrows the semantic distance among clients and facilitates mutual understanding.

To demonstrate the advantages of FedFoA on the final feature representation, we plot t-SNE distribution in Figure~\ref{tnse}.
We sample 2048 test samples of CIFAR-10 and feed them to the trained global model to extract the feature representations. After embedding the feature representations in a 2-dimensional space, we visualize the embedded data with a different color. Each color represents a category of the dataset. The Figure~\ref{poor} is the learned result of FedEMA, and Figure~\ref{good} is the learned result of the integrated algorithm with FedEMA and FedFoA. The feature representation of integrated algorithm shows more clear class boundaries. For instance, in the figure of integrated method, the green points are clustered together at the right-up corner. In contrast, the green points of FedEMA are scattered and it is hard to find the a boundary that can distinguish green points from other colors. 

\begin{figure}
     \centering
     \begin{subfigure}[b]{0.15\textwidth}
         \centering
         \includegraphics[width=\textwidth]{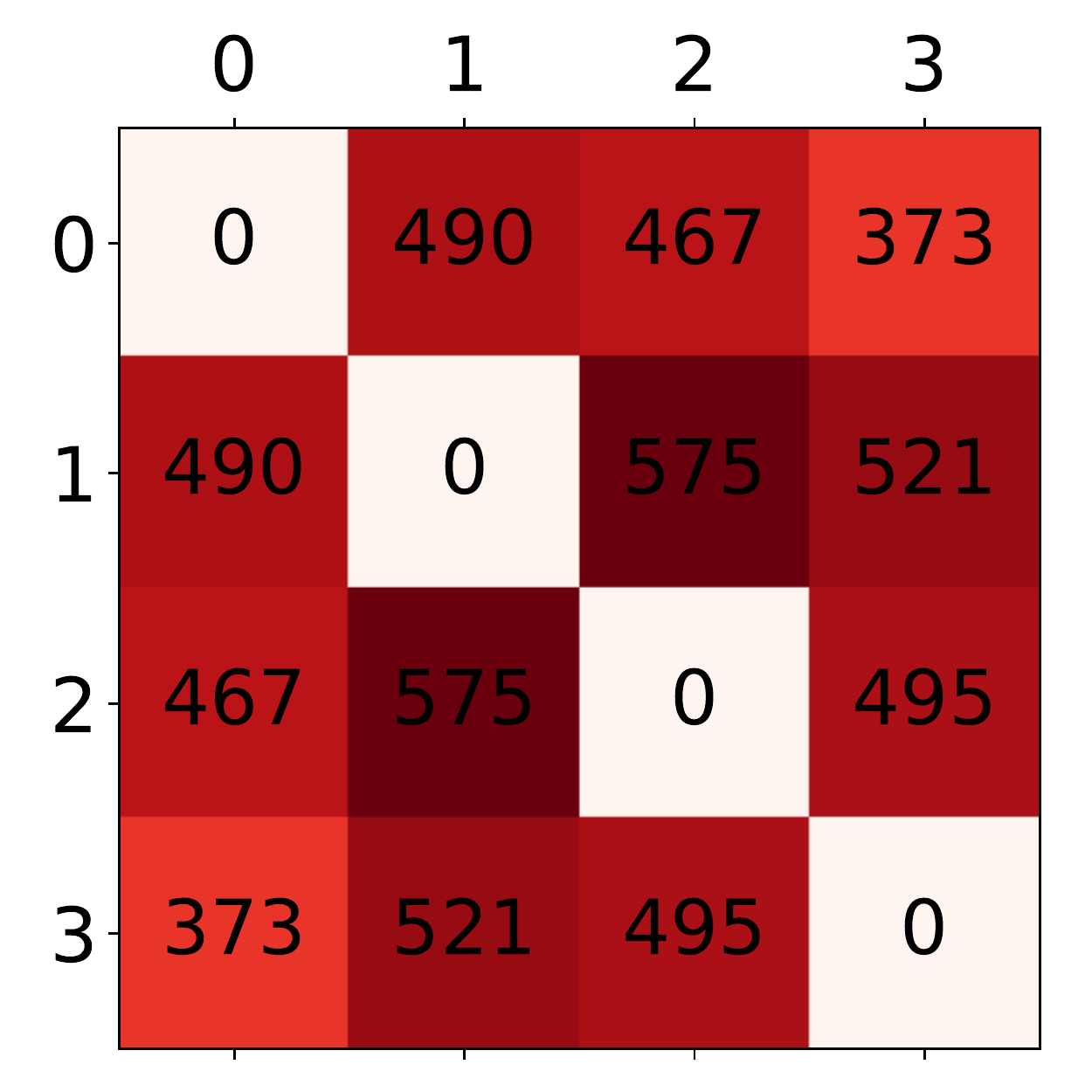}
         \caption{Round 0}
         \label{Round 0}
     \end{subfigure}
     \hfill
     \begin{subfigure}[b]{0.15\textwidth}
         \centering
         \includegraphics[width=\textwidth]{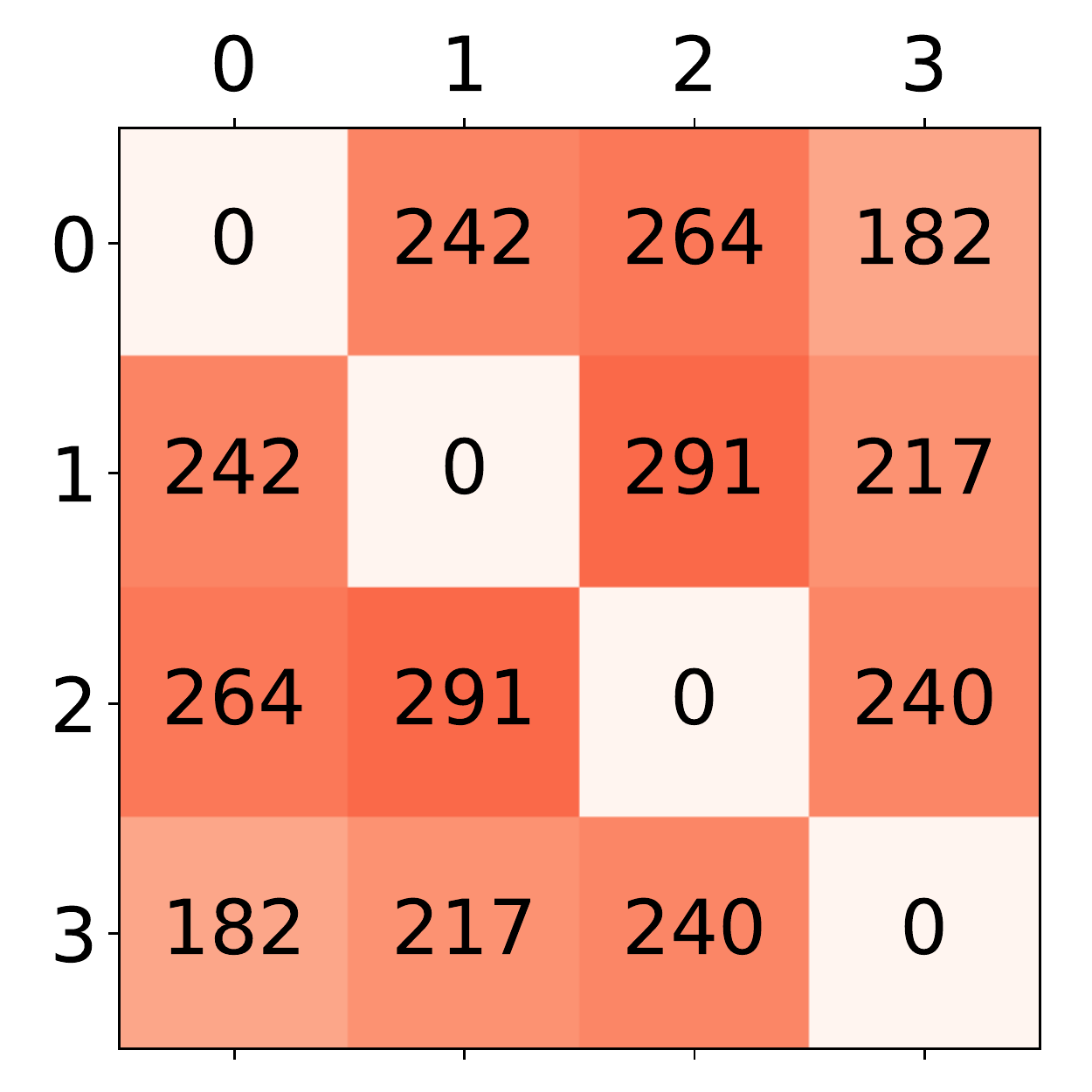}
         \caption{Round 29}
         \label{fig:three sin x}
     \end{subfigure}
     \hfill
     \begin{subfigure}[b]{0.15\textwidth}
         \centering
         \includegraphics[width=\textwidth]{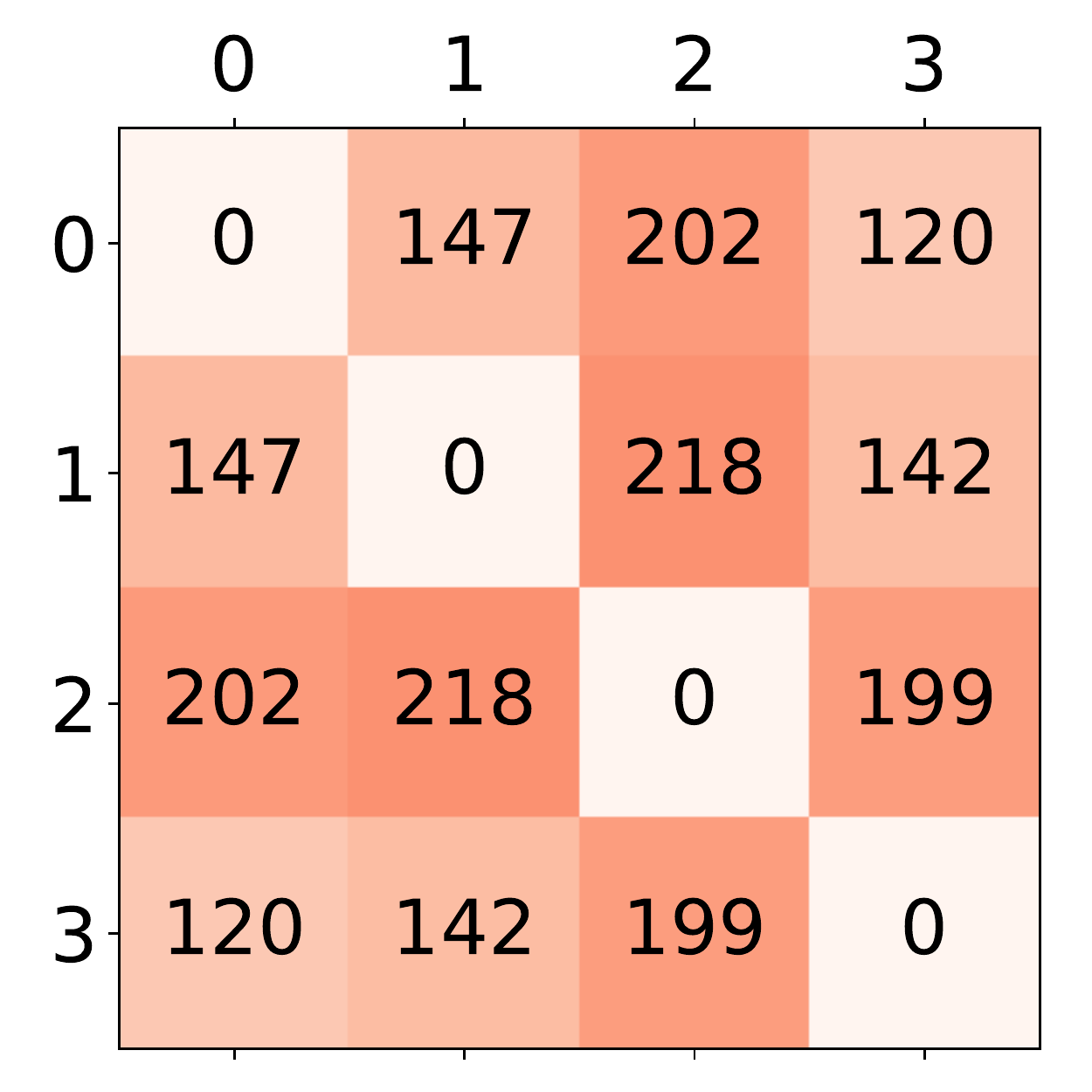}
         \caption{Round 49}
         \label{fig:five over x}
     \end{subfigure}
        \label{fig:three graphs}
    \caption{The feature correlation matrix distance. (a), (b) and (c) are three heat-map that represent the cross-client feature correlation distance at different training process. Heavier color means higher distance. }
    \vspace{-0.2cm}
    \label{heatmap}
\end{figure}

\begin{figure}
     \centering
     \begin{subfigure}[b]{0.236\textwidth}
         \centering
         \includegraphics[width=\textwidth]{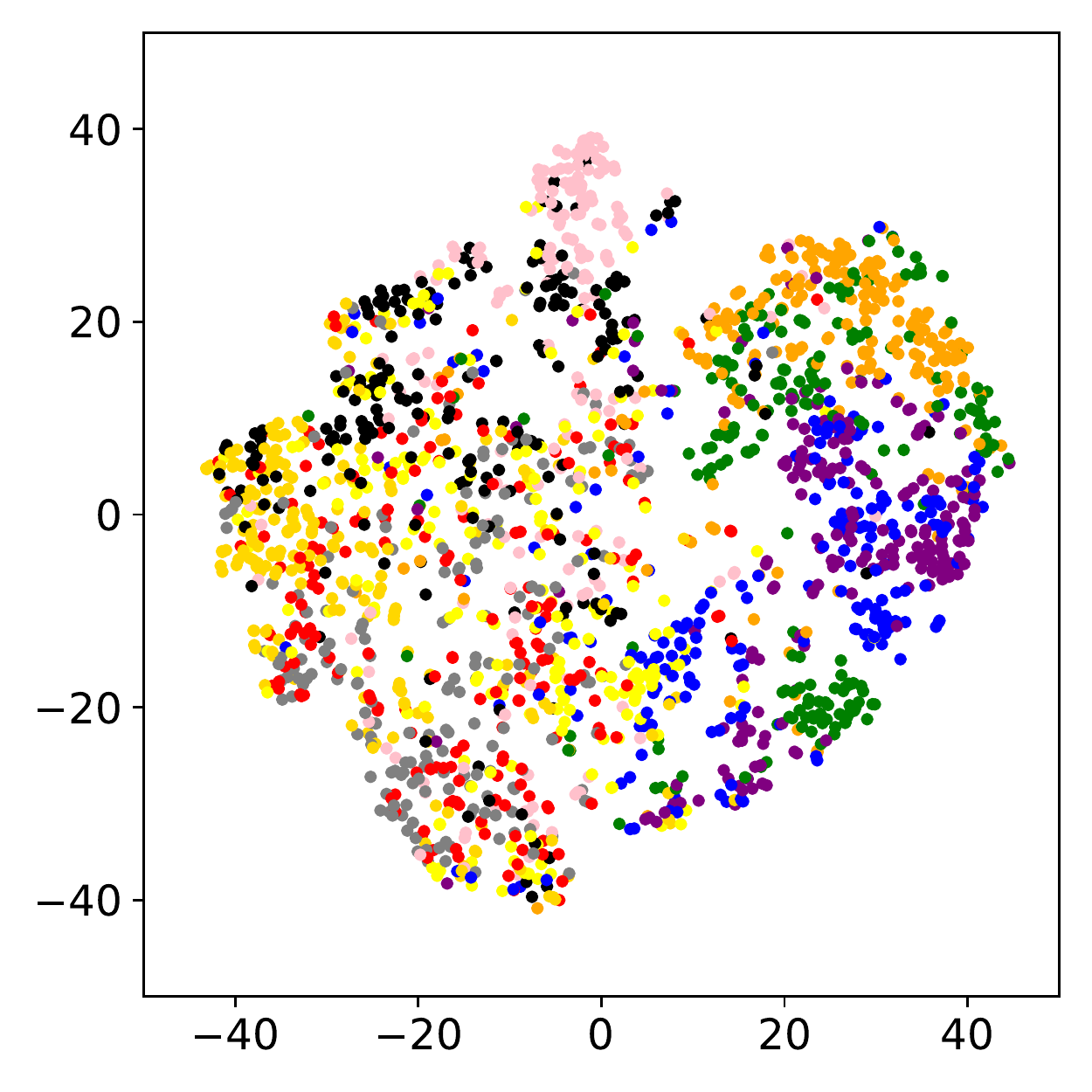}
         \caption{FedEMA}
         \label{poor}
     \end{subfigure}
     \hfill
     \begin{subfigure}[b]{0.236\textwidth}
         \centering
         \includegraphics[width=\textwidth]{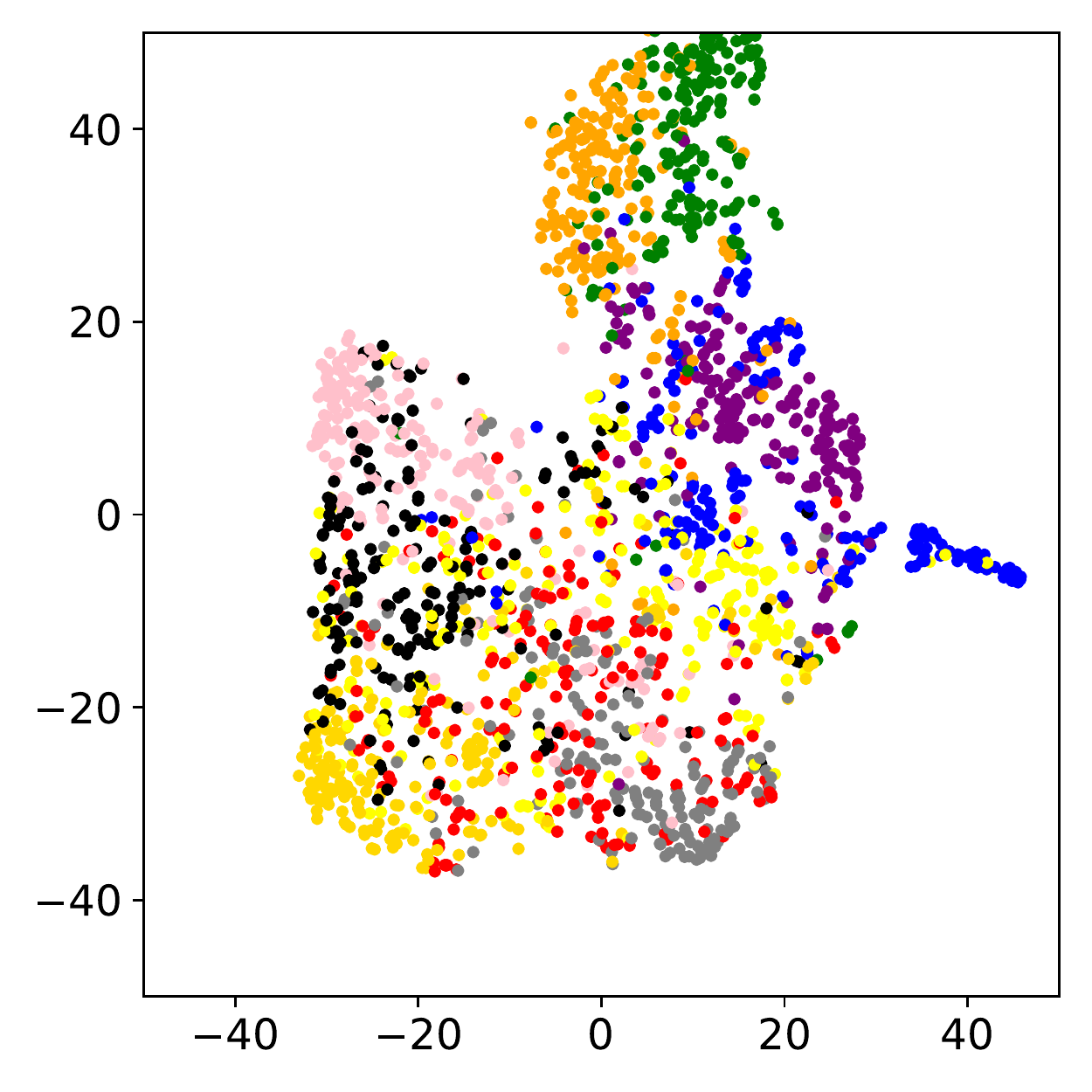}
         \caption{FedEMA+FedFoA}
         \label{good}
     \end{subfigure}
     \caption{Improvement brought by FedFoA. (a) and (b) are t-SNE visualization of learned representations in CIFAR-10 dataset at training round 49.}
     \label{tnse}
\vspace{-0.4cm}
\end{figure}


\vspace{-0.4cm}

\section{Conclusion}
In this paper, we propose a novel method called FedFoA, to achieve a data-free and model-agnostic collaborative training framework in a communication-efficient and privacy-preserving manner.
Specifically, we design a factorization-based method to extract shared feature correlation from local data and use it as a medium for knowledge transfer.
Then, the relation matrix can be regarded as a carrier of semantic information to perform the aggregation phase.
We prove that FedFoA is a general training framework and can be easily compatible with state-of-the-art unsupervised FL methods in a plug-and-play manner. 
Extensive empirical experiments on various benchmarks and downstream tasks demonstrate that FedFoA outperforms the state-of-the-art methods by a significant margin in terms of model accuracy and communication efficiency. 

{\small
\bibliographystyle{ieee_fullname}
\bibliography{PaperForReview}

\begin{thebibliography}{10}\itemsep=-1pt

\bibitem{bachman2019learning}
Philip Bachman, R~Devon Hjelm, and William Buchwalter.
\newblock Learning representations by maximizing mutual information across
  views.
\newblock {\em Advances in neural information processing systems}, 32, 2019.

\bibitem{chen2020simple}
Ting Chen, Simon Kornblith, Mohammad Norouzi, and Geoffrey Hinton.
\newblock A simple framework for contrastive learning of visual
  representations.
\newblock In {\em International conference on machine learning}, pages
  1597--1607. PMLR, 2020.

\bibitem{chen2021exploring}
Xinlei Chen and Kaiming He.
\newblock Exploring simple siamese representation learning.
\newblock In {\em Proceedings of the IEEE/CVF Conference on Computer Vision and
  Pattern Recognition}, pages 15750--15758, 2021.

\bibitem{dennis2021heterogeneity}
Don~Kurian Dennis, Tian Li, and Virginia Smith.
\newblock Heterogeneity for the win: One-shot federated clustering.
\newblock In {\em International Conference on Machine Learning}, pages
  2611--2620. PMLR, 2021.

\bibitem{gander1980algorithms}
Walter Gander.
\newblock Algorithms for the qr decomposition.
\newblock {\em Res. Rep}, 80(02):1251--1268, 1980.

\bibitem{gidaris2018unsupervised}
Spyros Gidaris, Praveer Singh, and Nikos Komodakis.
\newblock Unsupervised representation learning by predicting image rotations.
\newblock {\em arXiv preprint arXiv:1803.07728}, 2018.

\bibitem{goodfellow2020generative}
Ian Goodfellow, Jean Pouget-Abadie, Mehdi Mirza, Bing Xu, David Warde-Farley,
  Sherjil Ozair, Aaron Courville, and Yoshua Bengio.
\newblock Generative adversarial networks.
\newblock {\em Communications of the ACM}, 63(11):139--144, 2020.

\bibitem{grill2020bootstrap}
Jean-Bastien Grill, Florian Strub, Florent Altch{\'e}, Corentin Tallec, Pierre
  Richemond, Elena Buchatskaya, Carl Doersch, Bernardo Avila~Pires, Zhaohan
  Guo, Mohammad Gheshlaghi~Azar, et~al.
\newblock Bootstrap your own latent-a new approach to self-supervised learning.
\newblock {\em Advances in neural information processing systems},
  33:21271--21284, 2020.

\bibitem{han2022fedx}
Sungwon Han, Sungwon Park, Fangzhao Wu, Sundong Kim, Chuhan Wu, Xing Xie, and
  Meeyoung Cha.
\newblock Fedx: Unsupervised federated learning with cross knowledge
  distillation.
\newblock {\em arXiv preprint arXiv:2207.09158}, 2022.

\bibitem{he2020momentum}
Kaiming He, Haoqi Fan, Yuxin Wu, Saining Xie, and Ross Girshick.
\newblock Momentum contrast for unsupervised visual representation learning.
\newblock In {\em Proceedings of the IEEE/CVF conference on computer vision and
  pattern recognition}, pages 9729--9738, 2020.

\bibitem{hyeon2021fedpara}
Nam Hyeon-Woo, Moon Ye-Bin, and Tae-Hyun Oh.
\newblock Fedpara: Low-rank hadamard product for communication-efficient
  federated learning.
\newblock {\em arXiv preprint arXiv:2108.06098}, 2021.

\bibitem{jeong2020federated}
Wonyong Jeong, Jaehong Yoon, Eunho Yang, and Sung~Ju Hwang.
\newblock Federated semi-supervised learning with inter-client consistency \&
  disjoint learning.
\newblock {\em arXiv preprint arXiv:2006.12097}, 2020.

\bibitem{jin2020towards}
Yilun Jin, Xiguang Wei, Yang Liu, and Qiang Yang.
\newblock Towards utilizing unlabeled data in federated learning: A survey and
  prospective.
\newblock {\em arXiv preprint arXiv:2002.11545}, 2020.

\bibitem{komodakis2018unsupervised}
Nikos Komodakis and Spyros Gidaris.
\newblock Unsupervised representation learning by predicting image rotations.
\newblock In {\em International Conference on Learning Representations (ICLR)},
  2018.

\bibitem{li2019FedMD}
Daliang Li and Junpu Wang.
\newblock Fedmd: Heterogenous federated learning via model distillation.
\newblock {\em arXiv preprint arXiv:1910.03581}, 2019.

\bibitem{li2019fair}
Tian Li, Maziar Sanjabi, Ahmad Beirami, and Virginia Smith.
\newblock Fair resource allocation in federated learning.
\newblock In {\em International Conference on Learning Representations}, 2019.

\bibitem{lin2020ensemble}
Tao Lin, Lingjing Kong, Sebastian~U Stich, and Martin Jaggi.
\newblock Ensemble distillation for robust model fusion in federated learning.
\newblock {\em Advances in Neural Information Processing Systems},
  33:2351--2363, 2020.

\bibitem{lubana2022orchestra}
Ekdeep~Singh Lubana, Chi~Ian Tang, Fahim Kawsar, Robert~P Dick, and Akhil
  Mathur.
\newblock Orchestra: Unsupervised federated learning via globally consistent
  clustering.
\newblock {\em arXiv preprint arXiv:2205.11506}, 2022.

\bibitem{mcmahan2017communication}
Brendan McMahan, Eider Moore, Daniel Ramage, Seth Hampson, and Blaise~Aguera y
  Arcas.
\newblock Communication-efficient learning of deep networks from decentralized
  data.
\newblock In {\em Proceedings of Artificial Intelligence and Statistics
  (AISTATS)}, 2017.

\bibitem{mohri2019agnostic}
Mehryar Mohri, Gary Sivek, and Ananda~Theertha Suresh.
\newblock Agnostic federated learning.
\newblock In {\em International Conference on Machine Learning}, pages
  4615--4625. PMLR, 2019.

\bibitem{ng2021federated}
Dianwen Ng, Xiang Lan, Melissa Min-Szu Yao, Wing~P Chan, and Mengling Feng.
\newblock Federated learning: a collaborative effort to achieve better medical
  imaging models for individual sites that have small labelled datasets.
\newblock {\em Quantitative Imaging in Medicine and Surgery}, 11(2):852, 2021.

\bibitem{noroozi2016unsupervised}
Mehdi Noroozi and Paolo Favaro.
\newblock Unsupervised learning of visual representations by solving jigsaw
  puzzles.
\newblock In {\em European conference on computer vision}, pages 69--84.
  Springer, 2016.

\bibitem{oord2018representation}
Aaron van~den Oord, Yazhe Li, and Oriol Vinyals.
\newblock Representation learning with contrastive predictive coding.
\newblock {\em arXiv preprint arXiv:1807.03748}, 2018.

\bibitem{pathak2016context}
Deepak Pathak, Philipp Krahenbuhl, Jeff Donahue, Trevor Darrell, and Alexei~A
  Efros.
\newblock Context encoders: Feature learning by inpainting.
\newblock In {\em Proceedings of the IEEE conference on computer vision and
  pattern recognition}, pages 2536--2544, 2016.

\bibitem{sharma2013principal}
Alok Sharma, Kuldip~K Paliwal, Seiya Imoto, and Satoru Miyano.
\newblock Principal component analysis using qr decomposition.
\newblock {\em International Journal of Machine Learning and Cybernetics},
  4(6):679--683, 2013.

\bibitem{tan2021fedproto}
Yue Tan, Guodong Long, Lu Liu, Tianyi Zhou, Qinghua Lu, Jing Jiang, and Chengqi
  Zhang.
\newblock Fedproto: Federated prototype learning over heterogeneous devices.
\newblock {\em arXiv preprint arXiv:2105.00243}, 2021.

\bibitem{van2020towards}
Bram van Berlo, Aaqib Saeed, and Tanir Ozcelebi.
\newblock Towards federated unsupervised representation learning.
\newblock In {\em Proceedings of the third ACM international workshop on edge
  systems, analytics and networking}, pages 31--36, 2020.

\bibitem{vincent2008extracting}
Pascal Vincent, Hugo Larochelle, Yoshua Bengio, and Pierre-Antoine Manzagol.
\newblock Extracting and composing robust features with denoising autoencoders.
\newblock In {\em Proceedings of the 25th international conference on Machine
  learning}, pages 1096--1103, 2008.

\bibitem{wang2020federated}
Hongyi Wang, Mikhail Yurochkin, Yuekai Sun, Dimitris Papailiopoulos, and
  Yasaman Khazaeni.
\newblock Federated learning with matched averaging.
\newblock {\em arXiv preprint arXiv:2002.06440}, 2020.

\bibitem{zhang2020federated}
Fengda Zhang, Kun Kuang, Zhaoyang You, Tao Shen, Jun Xiao, Yin Zhang, Chao Wu,
  Yueting Zhuang, and Xiaolin Li.
\newblock Federated unsupervised representation learning.
\newblock {\em arXiv preprint arXiv:2010.08982}, 2020.

\bibitem{zhang2021edge}
Jie Zhang, Zhihao Qu, Chenxi Chen, Haozhao Wang, Yufeng Zhan, Baoliu Ye, and
  Song Guo.
\newblock Edge learning: The enabling technology for distributed big data
  analytics in the edge.
\newblock {\em ACM Computing Surveys (CSUR)}, 54(7):1--36, 2021.

\bibitem{zhang2016colorful}
Richard Zhang, Phillip Isola, and Alexei~A Efros.
\newblock Colorful image colorization.
\newblock In {\em European conference on computer vision}, pages 649--666.
  Springer, 2016.

\bibitem{zhuang2021collaborative}
Weiming Zhuang, Xin Gan, Yonggang Wen, Shuai Zhang, and Shuai Yi.
\newblock Collaborative unsupervised visual representation learning from
  decentralized data.
\newblock In {\em Proceedings of the IEEE/CVF International Conference on
  Computer Vision}, pages 4912--4921, 2021.

\bibitem{zhuang2021divergence}
Weiming Zhuang, Yonggang Wen, and Shuai Zhang.
\newblock Divergence-aware federated self-supervised learning.
\newblock In {\em International Conference on Learning Representations}, 2021.

\bibitem{zhuang2020performance}
Weiming Zhuang, Yonggang Wen, Xuesen Zhang, Xin Gan, Daiying Yin, Dongzhan
  Zhou, Shuai Zhang, and Shuai Yi.
\newblock Performance optimization of federated person re-identification via
  benchmark analysis.
\newblock In {\em Proceedings of the 28th ACM International Conference on
  Multimedia}, pages 955--963, 2020.

\end{thebibliography}
}

\end{document}